\documentclass[12pt]{article}

\usepackage[round]{natbib} 

\usepackage[margin=1.1in]{geometry}
\usepackage{lmodern}
\usepackage{parskip}

\usepackage[utf8]{inputenc} 
\usepackage[T1]{fontenc}    
\usepackage{hyperref}       
\usepackage{url}            
\usepackage{booktabs}       
\usepackage{amsfonts}       
\usepackage{nicefrac}       
\usepackage{microtype}      
\usepackage{hyperref}
\usepackage[pdftex]{graphicx}
\usepackage{xspace}
\usepackage{multirow}
\usepackage{subfigure}
\usepackage{enumitem}
\usepackage{algorithm}
\usepackage{algorithmic}
\usepackage{todonotes}
\usepackage{wrapfig}
\usepackage{cancel}

\usepackage{color}
\newcommand{\Bo}[1]{{\color{blue} [Bo: #1]}}

\newcommand{\estname}{{ALgorithm for policy Gradient from Arbitrary Experience via DICE}\xspace}
\newcommand{\estabb}{{AlgaeDICE}\xspace}
\title{
\estabb:
Policy Gradient \\ from Arbitrary Experience
}

\author{%
  \begin{tabular}{p{4.5cm}p{4.5cm}p{4.5cm}}
  {\hfil Ofir Nachum\thanks{Equal contribution.}}
    & {\hfil Bo Dai$^*$} & {\hfil Ilya Kostrikov\thanks{Also at NYU.}}\\[.1cm]
    \multicolumn{3}{c}{\hfil \normalfont \texttt{\{ofirnachum,~bodai,~kostrikov\}@google.com}}\\ [0.5cm]  
    {\hfil Yinlam Chow}
    & {\hfil Lihong Li} & {\hfil Dale Schuurmans\thanks{Also at University of Alberta.}}\\[.1cm]
    \multicolumn{3}{c}{\hfil \normalfont \texttt{\{yinlamchow,~lihong,~schuurmans\}@google.com}}\\ [.3cm]  
    \multicolumn{3}{c}{\normalfont Google Research}\\[0.1cm]
  \end{tabular} 
}
\date{}

\usepackage{Definitions}

\newcommand{\comment}[1]{}

\def\f{f}
\def\fdiv{D_{\f}}
\def\fstar{\f_*}
\def\E{\mathbb{E}}
\def\R{\mathbb{R}}
\def\qpi{Q_{\pi}}
\def\qhat{Q_\theta}
\def\jprimal{J_{\mathrm{P}}}
\def\jdual{J_{\mathrm{D}}}
\def\jcritic{J_{\mathrm{critic}}}
\def\jdicex{\tilde{J}_{\mathrm{D},\f}}
\def\jdice{J_{\mathrm{D},\f}}
\def\tr{\tilde{r}}
\def\mdp{\mathcal{M}}
\def\init{\mu_0}
\def\Init{\mathcal{U}}

\def\visitpi{d^\pi}

\def\visitrb{d^\Dset}

\def\bellman{\mathcal{B}}

\def\Wmax{W_{\mathrm{max}}}
\def\Rmax{R_{\mathrm{max}}}
\def\Zset{\mathcal{Z}}
\def\Nset{\mathcal{N}}
\def\Sset{S}
\def\Aset{A}
\def\SAset{\Sset\times\Aset}

\def\S{\Sset}
\def\A{\Aset}
\def\Dset{\mathcal{D}}
\def\Uset{\mathcal{U}}
\def\D{\Dset}
\def\defeq{:=}
\renewcommand{\widehat}{\hat}

\def\nustar{\nu^*}
\def\nustarmax{\nu^*_{\mathrm{max}}}
\def\xstar{x^*}
\def\lambdastar{\lambda^*}
\def\zetastar{\zeta^*}
\def\bellman{\mathcal{B}_{\pi}}
\def\bellmant{\mathcal{B}_{\pi}^{\top}}

\def\prob{\mathrm{Pr}}
\def\w{w_{\pi/\D}}

\def\tmix{T_{\mathrm{mix}}}
\newcommand{\tvar}[1]{\|#1\|_{\mathrm{TV}}}

\def\cite{\citep}

\begin{document}

\maketitle

\setlength{\abovedisplayskip}{2pt}
\setlength{\abovedisplayshortskip}{2pt}
\setlength{\belowdisplayskip}{2pt}
\setlength{\belowdisplayshortskip}{2pt}
\setlength{\jot}{2pt}

\setlength{\floatsep}{1ex}
\setlength{\textfloatsep}{1ex}

\vspace{-4mm}
\begin{abstract}
In many real-world applications of reinforcement learning (RL),
interactions with the environment are limited due to cost or feasibility.
This presents a challenge to traditional RL algorithms
since the max-return objective involves an expectation over on-policy samples.
We introduce a new formulation of max-return optimization that allows the
problem to be re-expressed by an expectation over an arbitrary
\emph{behavior-agnostic} and \emph{off-policy} data distribution.
We first derive this result by considering a regularized version of the {\em dual} max-return objective before extending our findings to unregularized objectives through the use of a Lagrangian formulation of the linear programming characterization of $Q$-values.
We show that, if auxiliary dual variables of the objective are optimized, then the gradient of the off-policy objective is \emph{exactly}
the on-policy policy gradient, without any use of importance weighting.
In addition to revealing the appealing theoretical properties of this approach,
we also show that it delivers good practical performance.
\end{abstract}

\section{Introduction}

The use of model-free reinforcement learning (RL) in conjunction with function approximation has proliferated in recent years, 
demonstrating successful applications in fields such as robotics~\citep{andrychowicz2018learning,nachum2019multi}, game playing~\citep{mnih2013playing}, and conversational systems~\citep{gao19neural}. 
These successes often rely on {\em on-policy} access to the environment; \ie, during the learning process agents may collect new experience from the environment using policies they choose, and these interactions are effectively unlimited.
By contrast, in many real-world applications of RL, interaction with the environment is costly, if not impossible, hence experience collection during learning is limited, necessitating the use of {\em off-policy} RL methods, \ie, algorithms which are able to learn from logged experience collected by potentially multiple and possibly unknown behavior policies.

The off-policy nature of many practical applications presents a significant challenge for RL algorithms. The traditional max-return objective is in the form of an on-policy expectation, and thus, policy gradient methods~\citep{sutton2000policy,konda00actor} require samples from the on-policy distribution to estimate the gradient of this objective.  
The most straightforward way to reconcile policy gradient with off-policy settings is via importance weighting~\citep{Precup00ET}. However, this approach is prone to high variance and instability without appropriate damping~\citep{munos2016safe,wang2016sample,gruslys2017reactor,schulman2017proximal}.  
The more common approach to the off-policy problem is to simply ignore it, which is exactly what has been proposed by many existing off-policy policy
gradient methods~\citep{degris2012off,silver2014deterministic}.  These algorithms simply compute the gradients of the max-return objective with respect to samples from the off-policy data, ignoring distribution shift in the samples.  The justification for this approach is that the maximum return policy will be optimal regardless of the sampling distribution of states.  However, such a justification is unsound in function approximation settings, where models have limited expressiveness, with potentially disastrous consequences on optimization and convergence~\citep[e.g.,][]{lu2018non}.

Value-based methods provide an alternative that may be more promising for the off-policy setting.  In these methods, a value function is learned either as a {\em critic} to a learned policy (as in actor-critic) or as the maximum return value function itself (as in $Q$-learning). This approach is based on dynamic programming in tabular settings, which is inherently off-policy and independent of any underlying data distribution.  Nevertheless, when using function approximation, the objective is
traditionally expressed as an expectation over single-step {\em Bellman errors}, which re-raises the question, ``What should the expectation be?'' Some theoretical work suggests that the ideal expectation is in fact the on-policy expectation~\citep{sutton2000policy, silver2014deterministic, smoothie}. 
In practice, this problem is usually ignored, with the same justification as that made for off-policy policy gradient methods. 
It is telling that actor-critic or $Q$-learning algorithms advertised as off-policy
still require large amounts of online interaction with the environment~\citep{HaaZhoAbbLev18,hessel2018rainbow}.

In this work, we present an \emph{\estname} (\estabb)\footnote{DICE is an abbreviation for distribution correction estimation and is taken from the DualDICE work~\citep{dualdice} on off-policy policy evaluation. Although our current work notably focuses on policy optimization as opposed to evaluation and only \emph{implicitly} estimates the distribution corrections, our derivations are nevertheless partly inspired by this previous work.} as an alternative to policy gradient and value-based methods. We start with the {\em dual} formulation of the max-return objective, which is expressed in terms of normalized state-action occupancies rather than a policy or value function.
Traditionally, this objective is considered unattractive, since access to the occupancies either requires an on-policy expectation (similar to policy gradient methods) or learning a function approximator to satisfy single-step constraints (similar to value-based methods).
We demonstrate how these problems can be remedied by adding a controllable regularizer and applying a carefully chosen change of variables, obtaining a joint objective over a policy and an auxiliary dual function (that can be interpreted as a critic).
Crucially, this objective relies only on access to samples from an arbitrary off-policy data distribution, collected by potentially multiple and possibly unknown behavior policies (under some mild conditions). 
Unlike traditional actor-critic methods, which use a separate objective for actor and critic, this formulation trains the policy (actor) and dual function (critic) to optimize the {\em same} objective.
Further illuminating the connection to policy gradient methods, we show that if the dual function is optimized, the gradient of the proposed objective with respect to the policy parameters is {\em exactly} the on-policy policy gradient. 
This way, our approach naturally avoids issues of distribution mismatch without any explicit use of importance weights.
We continue to provide an alternative derivation of the same results, based on a primal-dual form of the return-maximizing RL problem, and notably this perspective extends the previous results to both undiscounted $\gamma=1$ settings and unregularized max-return objectives.
Finally, we provide empirical evidence that \estabb can perform well on benchmark RL tasks.

\section{Background}
We consider the RL problem presented as a Markov Decision Process (MDP)~\cite{puterman1994markov}, which is specified by a tuple $\mdp = \langle \Sset, \Aset, r, T, \init \rangle$ consisting of a state space, an action space, a reward function, a transition probability function, and an initial state distribution.  A policy $\pi$ interacts with the environment by starting at an initial state $s_0 \sim \init$, and iteratively producing a sequence of distributions $\pi(\cdot|s_t)$ over $\Aset$, at
steps $t=0,1,...$, from which actions $a_t$ are sampled and successively applied to the environment.  At each step, the environment produces a scalar reward $r_t = r(s_t, a_t)$ and a next state $s_{t+1} \sim T(s_t, a_t)$.  
In RL, one wishes to learn a return-maximizing policy:
\begin{equation}\label{eq:primal}
\textstyle
  \max_{\pi} \jprimal(\pi) \defeq~(1-\gamma)~\E_{s_0\sim\init,a_0\sim\pi(s_0)}\left[\qpi(s_0,a_0)\right],
\end{equation}
where $\qpi$ describes the future rewards accumulated by $\pi$ from any state-action pair $(s,a)$,
\begin{equation}
\qpi(s,a) = \E\big[\sum_{t=0}^\infty \gamma^t r(s_t,a_t) ~\bigg|~s_0=s,a_0=a,s_t\sim T(s_{t-1},a_{t-1}),a_t\sim\pi(s_{t})\text{ for } t\ge1 \big],
\end{equation}
and $0\le \gamma<1$ is a {\em discount} factor.
This objective may be equivalently written in its {\em dual} form~\citep{puterman1994markov,WanLizBowSch08} in terms of the policy's normalized state visitation distribution as
\begin{equation}\label{eq:dual}
\textstyle
\max_{\pi} \jdual(\pi) \defeq~\E_{(s,a)\sim\visitpi}\left[r(s,a)\right],
\end{equation}
where
\begin{equation}\label{eq:stationary_dist}
\hspace{-0mm}
\visitpi(s,a)\hspace{-1mm} =\hspace{-1mm} (1-\gamma)\sum_{t=0}^\infty \gamma^t \prob\left[\left.
s_t=s,a_t=s\right| s_0\sim\init,a_t\sim\pi(s_t),s_{t+1}\sim T(s_t,a_t)\text{ for }t\ge0
\right].\hspace{-2mm}
\end{equation}
As we will discuss in~\secref{sec:q-lp} and~\appref{appendix:proof_details}, these objectives are the primal and dual of the same linear programming (LP) problem.  

In function approximation settings, optimizing $\pi$ requires access to gradients. The {\em policy gradient theorem}~\citep{sutton2000policy} provides the gradient of $\jprimal(\pi)$ as
\begin{equation}
  \label{eq:pg-thm}
\textstyle
\frac{\partial}{\partial\pi}\jprimal(\pi) = \E_{(s,a)\sim\visitpi}\left[\qpi(s,a)\nabla\log\pi(a|s)\right].
\end{equation}
To properly estimate this gradient one requires access to {\em on-policy} samples from $\visitpi$ and access to estimates of the $Q$-value function $\qpi(s,a)$.
The first requirement means that every gradient estimate of $\jprimal$ necessitates interaction with the environment, which limits applicability of this method in settings where interaction with the environment is expensive or infeasible.
The second requirement means that one must maintain estimates of the $Q$-function to learn $\pi$.  This leads to the family of {\em actor-critic} algorithms that alternate between updates to $\pi$ (the actor) and updates to a $Q$-approximator $\qhat$ (the critic). 
The critic is learned by encouraging it to satisfy single-step Bellman consistencies,
\begin{equation}
\qpi(s,a) = \bellman \qpi(s,a)\defeq r(s,a) + \gamma\cdot\E_{s'\sim T(s,a),a'\sim\pi(s')}\left[\qpi(s',a')\right],
\end{equation}
where $\bellman$ is the expected Bellman operator with respect to $\pi$. 
Thus, the critic is learned according to some variation on
\begin{equation}
\textstyle
\min_{\qhat} \jcritic(\qhat) \defeq \frac{1}{2}\E_{(s,a)\sim\beta}\left[(\bellman\qhat-\qhat)(s,a)^2\right],
\end{equation}
for some distribution $\beta$.
Although the use of an arbitrary $\beta$ suggests the critic may be learned off-policy, to achieve satisfactory performance, actor-critic algorithms generally rely on augmenting a replay buffer with {\em new} on-policy experience. Theoretical work has suggested that if one desires {\em compatible} function approximation, then an appropriate $\beta$ is, in fact, the on-policy distribution $\visitpi$~\citep{sutton2000policy, silver2014deterministic, smoothie}.

In this work, we focus on the off-policy setting directly.
Specifically, we are given a dataset $\D=\{(s_k,a_k,r_k,s_k')\}_{k=1}^N$, where $r_k=r(s_k,a_k)$; $s_k'\sim T(s_k,a_k)$; and $a_k$ has been sampled according to an unknown process.
We let $\visitrb$ denote the unknown state-action distribution, and additionally assume access to a sample of initial states, $\Init=\{s_{0,k}\}_{k=1}^N$, such that $s_{0,k}\sim\init$.

\section{\estabb via Density Regularization}\label{sec:dice-form}
We begin by presenting an informal derivation of our method, motivated as a regularization of the dual max-return objective in~\eqref{eq:dual}.
In Section~\ref{sec:q-lp} we will present our results more formally as a consequence of the Lagrangian of a linear programming formulation of the max-return objective.

\subsection{A Regularized Off-Policy Max-Return Objective}

The max-return objective~\eqref{eq:dual} is written exclusively in terms of the on-policy distribution $\visitpi$.
To introduce an off-policy distribution $\visitrb$ into the objective, we incorporate a regularizer:
\begin{equation}
  \label{eq:jdice}
\max_{\pi} \jdice(\pi) \defeq 
\E_{(s,a)\sim\visitpi}[r(s,a)]
-\alpha\fdiv(\visitpi \| \visitrb),
\end{equation}
with $\alpha>0$ and $\fdiv$ denoting the $\f$-divergence induced by a convex function $\f$:
\begin{equation}
    \fdiv(\visitpi \| \visitrb) = \E_{(s,a)\sim\visitrb}\left[\f\left(\w(s,a)\right)\right],
\end{equation}
where we have used the shorthand $\w(s,a)\defeq \frac{\visitpi(s,a)}{\visitrb(s,a)}$.
This form of regularization encourages {\em conservative} behavior, compelling
the state-action occupancies of $\pi$ to remain close to the off-policy distribution, which can improve generalization.
We emphasize that the introduction of this regularizer is to enable the subsequent derivations and \emph{not} to impose a strong constraint on the optimal policy.
Indeed, by appropriately choosing $\alpha$ and $f$, the strength of the regularization can be controlled. Later, we will show that many of our results also hold for \emph{exploratory} regularization ($\alpha<0$) and even for no regularization at all ($\alpha=0$).

At first glance, the regularization in~\eqref{eq:jdice} seems to complicate things. Not only do we still require on-policy samples from $\visitpi$, but we also have to compute $\fdiv(\visitpi\|\visitrb)$, which in general can be difficult. To make this objective more approachable, we transform the $\f$-divergence to its variational form~\citep{nguyen2010estimating} by use of a dual function $x:S\times A\to\R$:
\begin{align}
\max_{\pi} \min_{x:S\times A\to\R} & \jdicex(\pi,x) \nonumber \\
  & \defeq \E_{(s,a)\sim\visitpi}[r(s,a)]
    + \alpha \cdot \E_{(s,a)\sim\visitrb}[\fstar(x(s,a))] 
    -\alpha\cdot \E_{(s,a)\sim\visitpi}[x(s,a)] \nonumber \\
  & = \E_{(s,a)\sim\visitpi}[r(s,a)-\alpha\cdot x(s,a)]
    + \alpha\cdot \E_{(s,a)\sim\visitrb}[\fstar(x(s,a))]\,, \label{eq:dice-x}
\end{align}
where $\fstar$ is the convex (or Fenchel) conjugate of $f$.
With the objective in~\eqref{eq:dice-x}, we are finally ready to eliminate the expectation over on-policy samples from $\visitpi$. To do so, we make a change of variables, inspired by DualDICE~\citep{dualdice}. Define $\nu:S\times A \to \R$ as the fixed point of a variant of the Bellman equation,
\begin{equation}
    \label{eq:change-of-vars}
    \nu(s,a) \defeq -\alpha\cdot x(s,a) + \bellman\nu(s,a).
\end{equation} 
Equivalently, $x(s,a) = \frac{1}{\alpha}(\bellman\nu-\nu)(s,a)$.  Note that $\nu$ always exists and is bounded when $x$ and $r$ are bounded~\citep{puterman1994markov}.
Applying this change of variables to~\eqref{eq:dice-x} (after some telescoping, see~\citet{dualdice}) yields 
\begin{align}
  \max_{\pi} \min_{\nu:S\times A\to\R} & \jdice(\pi,\nu) \nonumber \\
    & \defeq (1-\gamma)\E_{\substack{s_0\sim\init \\ a_0\sim\pi(s_0)}}[\nu(s_0,a_0)]
      + \alpha\cdot\E_{(s,a)\sim\visitrb}[\fstar((\bellman\nu - \nu)(s,a)/\alpha)]\,.   \label{eq:dice-nu} 
\end{align}
The resulting objective is now completely off-policy, relying only on access to samples from the initial state distribution $\init$ and the off-policy dataset $\visitrb$.
Thus, we have our first theorem, providing an off-policy formulation of the max-return objective:
\begin{theorem}[Primal \estabb]
  \label{thm:max-return}
Under mild conditions on $\visitrb,\alpha,f$, the regularized max-return objective may be expressed as a max-min optimization:
\begin{equation}
\label{eq:thm1}
\max_{\pi} ~\E_{(s,a)\sim\visitpi}[r(s,a)] - \alpha\fdiv(\visitpi \| \visitrb) \equiv \max_{\pi} \min_{\nu:S\times A\to\R} \jdice(\pi,\nu)\,.
\end{equation}
\end{theorem}
\paragraph{Remark (extension to $\alpha < 0$):}
It is clear that the same derivations above may apply to an exploratory regularizer of the same form 
with $\alpha<0$, which leads to the following optimization problem:
  \begin{equation}
    \max_\pi\max_{\nu:S\times A\to\R} ~(1-\gamma)\cdot\E_{\substack{s_0\sim\init \\ a_0\sim\pi(s_0)}}[\nu(s_0,a_0)] + \alpha\cdot\E_{(s,a)\sim\visitrb}[\fstar((\bellman\nu - \nu)(s,a) / \alpha)].
  \end{equation}
\paragraph{Remark (Fenchel \estabb)} 
The appearance of $\bellman$ inside $\fstar$ in the second term of~\eqref{eq:dice-nu} presents a challenge in practice, since $\bellman$ involves an expectation over the transition function $T$, whereas one typically only has access to a single empirical sample from $T$ for a given state-action pair.  This challenge, known as double sampling in the RL literature~\citep{baird1995residual}, can prevent the algorithm from finding the desired value function, even with infinite data.
There are several alternatives to handle this issue~\citep[e.g.,][]{antos2006learning,farahmand16regularized,feng19kernel}.
Here, we apply the dual embedding technique~\citep{DaiHePanBooetal16,dai18sbeed}.
Specifically, the dual representation of $\fstar$,
$$
    \fstar((\bellman\nu - \nu)(s,a)/\alpha) = \max_{\zeta\in\R} ~\frac{1}{\alpha}(\bellman\nu - \nu)(s,a) \cdot \zeta - f(\zeta),
$$
can be substituted into~\eqref{eq:dice-nu}, to result in a $\max$-$\min$-$\max$ problem:
\begin{align}
& \max_{\pi} ~\E_{(s,a)\sim\visitpi}[r(s,a)]
-\alpha\fdiv(\visitpi \| \visitrb) \nonumber \\
= & 
\max_{\pi} \min_{\nu:S\times A\to\R}\max_{\zeta:\S\times\A\to\R} ~
        (1-\gamma)\E_{s_0\sim\init, a_0\sim\pi(s_0)}[\nu(s_0,a_0)] + \nonumber \\
 & ~~~~ \E_{(s,a)\sim\visitrb,s'\sim T(s,a),a'\sim\pi(s')}[(\gamma \nu(s',a') + r(s,a) - \nu(s,a)) \cdot \zeta(s,a) - \alpha\cdot\f(\zeta(s,a))]\,. \label{eq:fenchel-dice-nu}
\end{align}
As we will see in Section~\ref{sec:q-lp}, under mild conditions, strong duality holds in the inner $\min$-$\max$ of~\eqref{eq:fenchel-dice-nu}, hence one can switch the $\min_\nu$ and $\max_\zeta$ to reduce to a more convenient $\max$-$\max$-$\min$ form. 


\subsection{Consistent Policy Gradient using Off-Policy Data}

The equivalence between the objective in~\eqref{eq:dice-nu} and the on-policy max-return objective can be highlighted by considering the gradient of this objective with respect to $\pi$.
First, consider the optimal $\xstar_\pi \defeq \argmin_x \jdice(\pi, x)$ for~\eqref{eq:dice-x}. By taking the gradient of $\jdice$ with respect to $x$ and setting this to 0, one finds that $\xstar_\pi$ satisfies
\begin{equation}
\label{eq:xstar}
\fstar'(\xstar_\pi(s,a)) = \w(s,a).
\end{equation}
Accordingly, for any $\pi$, the optimal $\nustar_\pi\defeq \argmin_\nu \jdice(\pi, \nu)$ for~\eqref{eq:dice-nu} satisfies
\begin{equation}
\fstar'((\bellman\nustar_\pi - \nustar_\pi)(s,a) / \alpha) = \w(s,a).
\end{equation}
Thus, we may express the gradient of $\jdice(\pi, \nustar_\pi)$ with respect to $\pi$ as
\begin{align*}
\frac{\partial}{\partial\pi} \jdice(\pi, \nustar_\pi) 
&= (1-\gamma)\frac{\partial}{\partial\pi}\E_{\substack{a_0\sim\pi(s_0) \\ s_0\sim\init}}[\nustar_\pi(s_0,a_0)]
        + \E_{(s,a)\sim\visitrb}\left[\w(s,a) \frac{\partial}{\partial\pi}(\bellman\nustar_\pi - \nustar_\pi)(s,a)\right] \\
&=(1-\gamma)\frac{\partial}{\partial\pi}\E_{\substack{a_0\sim\pi(s_0) \\ s_0\sim\init}}[\nustar_\pi(s_0,a_0)]
+ \gamma\cdot \E_{\substack{(s,a)\sim\visitpi,\\ s'\sim T(s,a)}}\left[ \frac{\partial}{\partial\pi} \E_{a'\sim\pi(s')}[\nustar_\pi(s',a')]\right] \\
&= \E_{(s,a)\sim\visitpi}\left[\nustar_\pi(s,a) \nabla \log\pi(a|s)\right],
\end{align*}
where we have used Danskin's theorem~\citep{Bertsekas99} to ignore gradients of $\pi$ through $\nustar_\pi$. Hence, if the dual function $\nu$ is optimized, the gradient of the off-policy objective $\jdice(\nu,\pi)$ is \emph{exactly} the on-policy policy gradient, with $Q$-value function given by $\nustar_\pi$. 

To characterize this $Q$-value function, note that
from~\eqref{eq:change-of-vars}, $\nustar_\pi$ is a $Q$-value function with respect to augmented reward $\tr(s,a)\defeq r(s,a) - \alpha\cdot \xstar_\pi(s,a)$.
Recalling the expression for $\xstar_\pi$ in~\eqref{eq:xstar} and the fact that the derivatives $\f'$ and $\fstar'$ are inverses of each other, we have, $\tr(s,a) = r(s,a) - \alpha\cdot \f'(\w(s,a))$.
This derivation leads to our second theorem:
\begin{theorem}
  \label{thm:on-policy}
If the dual function $\nu$ is optimized, the gradient of the off-policy objective $\jdice(\pi, \nu)$ with respect to $\pi$ is the regularized on-policy policy gradient:
\begin{equation} \label{eqn:offpolicy-gradient}
\frac{\partial}{\partial\pi} \min_{\nu} \jdice(\pi, \nu) =
\E_{(s,a)\sim\visitpi}\left[\tilde{Q}_\pi(s,a) \nabla \log\pi(a|s)\right],
\end{equation}
where, $\tilde{Q}_\pi(s,a)$ is the $Q$-value function of $\pi$ with respect to rewards $\tr(s,a) \defeq r(s,a) - \alpha\cdot\f'(\w(s,a))$.
\end{theorem}

\paragraph{Remark}
We note that Theorem~\ref{thm:on-policy} also holds when using the more sophisticated objective in~\eqref{eq:fenchel-dice-nu}, since the optimal $\zetastar_\pi$ is equal to $\w$, regardless of $\pi$. 

\subsection{Connection to Actor-Critic}
The relationship between the proposed off-policy objective and the 
classic policy gradient becomes more profound when we consider the form of the objective under specific choices of convex function $\f$.
If we take $f(x)=\frac{1}{2}x^2$, then $\fstar(x)=\frac{1}{2}x^2$ and the proposed objective is reminiscent of actor-critic:
\begin{equation*}
\label{eq:dice-sq}
\max_{\pi} \min_{\nu:S\times A\to\R} \jdice(\pi,\nu) \defeq  (1-\gamma)\E_{s_0\sim\init,a_0\sim\pi(s_0)}[\nu(s_0,a_0)] + \frac{1}{2\alpha}\cdot\E_{(s,a)\sim\visitrb}[((\bellman\nu-\nu)(s,a))^2]. 
\end{equation*}
The second term alone is an instantiation of the off-policy critic objective in actor-critic. 
However, in actor-critic, the use of an off-policy objective for the critic is difficult to theoretically motivate.  
Moreover, in practice, critic and actor learning can both suffer from the mismatch between the off-policy distribution $\visitrb$ and the on-policy $\visitpi$.
By contrast, our derivations show that the introduction of the first term to the objective transforms the off-policy actor-critic algorithm to an on-policy actor-critic, without any explicit use of importance weights. 
Moreover, while standard actor-critic has two separate objectives for
value and policy, our proposed objective is a single, unified objective. Both the policy and value functions are trained with respect to the same off-policy objective.

\section{A Lagrangian View of~\estabb}\label{sec:q-lp}

We now show how \estabb can be alternatively derived from the Lagrangian of a linear programming~(LP) formulation of the $\qpi$-function. Please refer to~\appref{appendix:proof_details} for details. 
We begin by introducing some notations and assumptions, which have appeared in the literature~\citep[e.g.,][]{puterman1994markov,dualdice,zhang2020gendice}.
\begin{assumption}[Bounded rewards]\label{assumption:bounded_rewards}
  The rewards of the MDP are bounded by some finite constant $\Rmax$: $\nbr{r}_\infty\le \Rmax$.
\end{assumption}
For the next assumption, we introduce the \emph{transpose Bellman operator}:
\begin{equation}
\bellmant\rho\rbr{s', a'} \defeq \gamma\sum_{s, a}{\pi\rbr{a'|s'}T\rbr{s'|s, a}\rho\rbr{s, a}} + \rbr{1 -\gamma}\mu_0\rbr{s'}\pi\rbr{a'|s'}.
\end{equation}
\begin{assumption}[MDP regularity]\label{assumption:mdp_reg}
  The transposed Bellman operator $\bellmant$
  has a unique fixed point solution.\footnote{
When $\gamma\in [0, 1)$, $\bellmant$ has a unique fixed point regardless of the underlying MDP. 
For $\gamma=1$, in the discrete case, the assumption reduces to requiring that $\bellmant$ be ergodic. The continuous case for $\gamma=1$ is more involved; see~\citet{MeyTwe12,LevPer17} for a detailed discussion.}
\end{assumption}
\begin{assumption}[Bounded ratio]\label{assumption:bounded_ratios}
The target density ratio is bounded by some finite constant $\Wmax$: $\nbr{\w}_\infty\le \Wmax$.
\end{assumption}
\begin{assumption}[Characterization of $f$]\label{assumption:convex_f}
  The function $\f$ is convex with domain $\R$ and continuous derivative $\f'$. 
  The convex (Fenchel) conjugate of $\f$ is $\fstar$, and $\fstar$ is closed and {strictly convex}.  
  The derivative $\fstar'$ is continuous and its range is a superset of $[0,\Wmax]$.
\end{assumption}
For convenience, we define $\Nset \defeq [-\frac{C}{1-\gamma}, \frac{C}{1-\gamma}]$ where $C=\Rmax + |\alpha| \cdot \f'(\Wmax)$.  $\Nset$ will serve as the range for $\nu$.

Our derivation begins with a formalization of the LP characterization of the $\qpi$-function and its dual form:
\begin{theorem}\label{theorem:q_lp}
  Given a policy $\pi$,
  the average return of $\pi$ may be expressed in the primal and dual forms as \\
  \begin{minipage}{0.45\textwidth}
\begin{align}\label{eq:q_lp}
\min_{\nu:\S\times\A\rightarrow \Nset}& \jprimal(\pi, \nu) \defeq \rbr{1-\gamma}\EE_{\mu_0\pi}\sbr{\nu\rbr{s_0, a_0}}\nonumber \\
  \st~~~~ &\nu\rbr{s, a}\ge \bellman\nu(s, a), \\
  & \forall \rbr{s, a}\in S\times A,\nonumber
\end{align}
\end{minipage}
~and,
\begin{minipage}{0.45\textwidth}
\begin{align} \label{eq:rho_lp}
  \max_{\rho:\S\times\A\rightarrow \RR_+}& \jdual(\pi, \rho) \defeq \EE_{\rho}\sbr{r\rbr{s, a}}\nonumber\\
  \st~~~~ &\rho\rbr{s, a} =  \bellmant\rho(s, a), \\
  & \forall \rbr{s, a}\in S\times A,\nonumber 
\end{align}
\end{minipage}

  respectively. Under Assumptions~\ref{assumption:bounded_rewards} and \ref{assumption:mdp_reg}, strong duality holds, \ie, $\jprimal(\pi, \nustar_\pi) = \jdual(\pi, \rho^*_\pi)$ for optimal solutions $\nustar_\pi,\rho^*_\pi$. 
  The optimal primal satisfies $\nustar_\pi\rbr{s, a} = \qpi\rbr{s, a}$ for all $\rbr{s, a}$ reachable by $\pi$ 
  and the optimal dual $\rho^*_\pi$ is $\visitpi$.
\end{theorem}

Consider the Lagrangian of $\jprimal$, which would typically be expressed with a sum (or integral) of constraints weighted by $\rho$. We can reparametrize the dual variable as $\zeta\rbr{s, a} = \frac{\rho\rbr{s, a}}{\visitrb\rbr{s, a}}$ to express the Lagrangian as,
\begin{equation}\label{eq:q_lagrangian}
  \min_{\nu:\S\times\A\rightarrow \Nset}\max_{\zeta:\S\times\A\rightarrow \RR_+} \,\, \rbr{1 - \gamma}\EE_{(s_0, a_0)\sim\init\pi}\sbr{\nu\rbr{s_0, a_0}} + \EE_{(s,a)\sim\visitrb}\sbr{\zeta\rbr{s, a} {\rbr{\bellman \nu - \nu}\rbr{s, a}}}.
  \hspace{-2mm}
\end{equation}
The optimal $\zetastar_\pi$ of this Lagrangian is $\w$,
and this optimal solution is not affected by expanding the allowable range of $\zeta$ to all of $\R$.
However, in practice, the linear structure in~\eqref{eq:q_lagrangian} can induce numerical instability. Therefore, inspired by the augmented Lagrangian method, we introduce regularization. By adding a special regularizer $\alpha\cdot\EE_{\visitrb}\sbr{f\rbr{\zeta\rbr{s, a}}}$ using convex $f$, we obtain
\begin{multline}\label{eq:q_lagrangian_reg}
  \min_{\nu:\Sset\times\Aset\to\Nset}\max_{\zeta:\Sset\times\Aset\to\RR} L\rbr{\nu, \zeta; \pi}\defeq\rbr{1 - \gamma}\EE_{(s_0,a_0\sim\init\pi}\sbr{\nu\rbr{s_0, a_0}} + \\ \EE_{(s,a)\sim\visitrb}\sbr{\zeta\rbr{s, a} {\rbr{\bellman \nu - \nu}\rbr{s, a}}} - \alpha\cdot\EE_{(s,a)\sim\visitrb}\sbr{f\rbr{\zeta\rbr{s, a}}}.
\end{multline}
We characterize the optimizers $\nustar_\pi$ and $\zetastar_\pi$ and the optimum value
$L(\nustar_\pi,\zetastar_\pi;\pi)$ of this objective in the following theorem.
Interestingly, although the regularization can affect the optimal primal solution $\nustar_\pi$, the optimal dual solution $\zetastar_\pi$ is unchanged.
\begin{theorem}\label{thm:effect_reg}
Under Assumptions~\ref{assumption:bounded_rewards}--\ref{assumption:convex_f},
the solution to~\eqref{eq:q_lagrangian_reg} is given by,
\begin{eqnarray*}
  \nustar_\pi\rbr{s, a} &=& -\alpha f'\rbr{\w\rbr{s, a}}+\bellman\nustar_\pi\rbr{s,a}, \\
  \zetastar_\pi\rbr{s, a} &=& \w\rbr{s, a}.
\end{eqnarray*}
The optimal value is $L\rbr{\nustar_\pi, \zetastar_\pi;\pi} = \E_{\visitpi}[r(s,a)]-\alpha\fdiv(\visitpi \| \visitrb)$. 
\end{theorem}
Thus, we have recovered the Fenchel \estabb objective for $\pi$, given in Equation~\ref{eq:fenchel-dice-nu}.
Furthermore, one may reverse the Legendre transform, $\fstar((\bellman\nu - \nu)(s,a)/\alpha) = \max_{\zeta}~\frac{1}{\alpha}(\bellman\nu - \nu)(s,a) \cdot \zeta - f(\zeta)$, to recover the Primal \estabb objective in Equation~\eqref{eq:thm1}.

The derivation of this same result from the LP perspective allows us to exploit strong duality. Specifically, under the assumption that $\w$ and $r$ are bounded, $\rbr{\nustar_\pi, \zetastar_\pi}$ does not change if we optimize $L\rbr{\nu, \zeta;\pi}$ over a bounded space $\Hcal\times\Fcal$, as long as $\rbr{\nustar_\pi, \zetastar_\pi}\in \Hcal\times\Fcal$.
In this case, strong duality holds~\citep[Proposition~2.1]{EkeTem99}, and we obtain
$$
\min_{\nu\in\Hcal}\max_{\zeta\in\Fcal}L\rbr{\nu, \zeta; \pi} = \max_{\zeta\in\Fcal}\min_{\nu\in\Hcal} L\rbr{\nu, \zeta; \pi}.
$$    
This implies that, for computational efficiency, we can optimize the policy via
\begin{multline}\label{eq:q_lagrangian_reg_opt}
  \max_{\pi}\ell\rbr{\pi} \defeq  \max_{\zeta\in \Fcal}\min_{\nu\in\Hcal} \rbr{1 - \gamma}\EE_{(s_0,a_0)\sim\init\pi}\sbr{\nu\rbr{s_0, a_0}} + \\ \EE_{\visitrb}\sbr{\zeta\rbr{s, a} {\rbr{\bellman \nu -\nu}\rbr{s, a}}} - \alpha\cdot\EE_{\visitrb}\sbr{f\rbr{\zeta\rbr{s, a}}}.
\end{multline}

\paragraph{Remark (extensions to $\gamma = 1$ or $\alpha = 0$):}
Although~\estabb~is originally derived for $\gamma\in[0, 1)$ and $\alpha>0$ in~\secref{sec:dice-form}, the Lagrangian view of the LP formulation of $\qpi$ can be used to generalize the algorithm to $\gamma=1$ and $\alpha=0$. In particular, for $\alpha = 0$, one can directly use the original Lagrangian for the LP. For the case $\gamma =1$, the problem reduces to the Lagrangian of the LP for an undiscounted $\qpi$-function; details are delegated to \appref{appendix:proof_details}.


\paragraph{Remark (off-policy evaluation):}
The LP form of the $Q$-values leading to the Lagrangian~\eqref{eq:q_lagrangian} can be directly used for \emph{behavior-agnostic off-policy evaluation}~(OPE). In fact, existing estimators for OPE in the behavior-agnostic setting which typically reduce the OPE problem to estimation of quantities $\w$ (\eg, DualDICE~\citep{dualdice} and GenDICE~\citep{zhang2020gendice}) can be recast as special cases by introducing different regularizations to the Lagrangian. As we have shown, the solution to the Lagrangian provides both (regularized)
$Q$-values and the desired state-action corrections $\w$ as primal and dual variables simultaneously.

\section{Related Work}

Algorithmically, our proposed method follows a Lagrangian primal-dual view of the LP characterization of the $Q$-function, which leads to a saddle-point problem. 
Several recent works~\citep[e.g.,][]{ChenWang16,Wang17,DaiShaHeLietal18,dai18sbeed,CheLiWan18,LeeNiao18} also considered saddle-point formulations for policy improvement, derived from fundamentally different perspectives. 
In particular,~\citet{DaiShaHeLietal18} exploit a saddle-point formulation for the
multi-step (path) conditions on the consistency between optimal value function and policy. 
Other works~\citep{ChenWang16,Wang17,dai18sbeed,CheLiWan18} consider the (augmented) Lagrangian of the LP characterization of Bellman optimality for the optimal $V$-function, which is slightly different from the LP characterization with respect to the optimal $Q$-function we consider. 
Although slight, the difference between the $V$- and $Q$-LPs is crucial to enable {\em behavior-agnostic} policy optimization in~\estabb. If one were to follow derivations similar to~\estabb but for the $V$-function LP, some form of explicit importance weighting (and thus knowledge of the behavior policy) would be required, as in recent work on off-policy \emph{estimation}~\citep{tang19doubly,uehara19minimax}.
We further note that the application of a regularizer on the dual variable to yield Primal~\estabb is key to transforming the Lagrangian optimization over values and state-action occupancies --- typical in these previous works --- to an optimization over values and policies, which is more common in practice and can help generalization~\citep[e.g.,][]{swaminathan15batch}.

The regularization we employ is inspired by previous uses of regularization in RL.
Adding regularization to MDPs~\citep{NeuJonGom17,geist2019theory}
has been investigated for many different purposes
in the literature,
including exploration~\citep{Defarias00,HaaTanAbbLev17, HaaZhoAbbLev18},
smoothing~\citep{dai18sbeed},
avoiding premature convergence~\citep{NacNorXuSch17},
ensuring tractability~\citep{Todorov07},
and mitigating observation noise~\citep{RubShaTis12,FoxPakTis15}.
We note that the regularization employed by \estabb as a divergence over state-action densities is markedly different from these previous works, which mostly regularize only the {\em action} distributions of a policy conditioned on state.
An approach more similar to ours is given by~\citet{belousov2017f},
which regularizes the max-return objective using an $f$-divergence over state-action densities.
Their derivations are similar in spirit to ours, using the method of Lagrange multipliers, but their result is distinct in a number of key characteristics. 
First, their objective (analogous to ours in~\eqref{eq:dice-nu}) includes not only policy and values but also a number of additional functions, complicating any practical implementation.
Second, their results are restricted to conservative regularization ($\alpha>0$), whereas our findings extend to both exploratory regularization and unregularized objectives ($\alpha\le0$).
Third, the algorithm proposed by~\citet{belousov2017f} follows a bi-level optimization, in which the policy is learned using a separate and distinct objective.
In contrast, our proposed~\estabb uses a single, unified objective for both policy and value learning. 

Lastly, there are a number of works which (like ours) perform policy gradient on off-policy data via distribution correction.
The key differentiator is in how the distribution corrections are computed.
One common method is to re-weight off-policy samples by considering eligibility traces~\citep{Precup00ET,geist2014off}, i.e., compute weights by taking the product of per-action importance weights over a trajectory.
Thus, these methods can suffer from high variance as the length of trajectory increases, known as the ``curse of horizon''~\citep{liu2018breaking}.
A more recent work~\citep{liu19off} attempts to weight updates by estimated state-action distribution corrections. This is more in line with our proposed~\estabb, which implicitly estimates these quantities. One key difference is that this previous work explicitly estimates these corrections, which results in a bi-level optimization, as opposed to our more appealing unified objective.
It is also important to note that both eligibility trace methods and the technique outlined in~\citet{liu19off} require knowledge of the behavior policy. In contrast,~\estabb is a behavior-agnostic off-policy policy gradient method, which may be more relevant in practice.  Compared to existing behavior-agnostic off-policy estimators~\citep{dualdice,zhang2020gendice}, this work considers the substantially more challenging problem of policy optimization.

\section{Experiments}

We present empirical evaluations of \estabb, first in a tabular setting using the Four Rooms domain~\citep{sutton1999between} and then on a suite of continuous control benchmarks using MuJoCo~\citep{TodEreTas12} and OpenAI Gym~\citep{brockman2016openai}.

\begin{figure}[t]
  \begin{center}
  \includegraphics[width=0.99\textwidth]{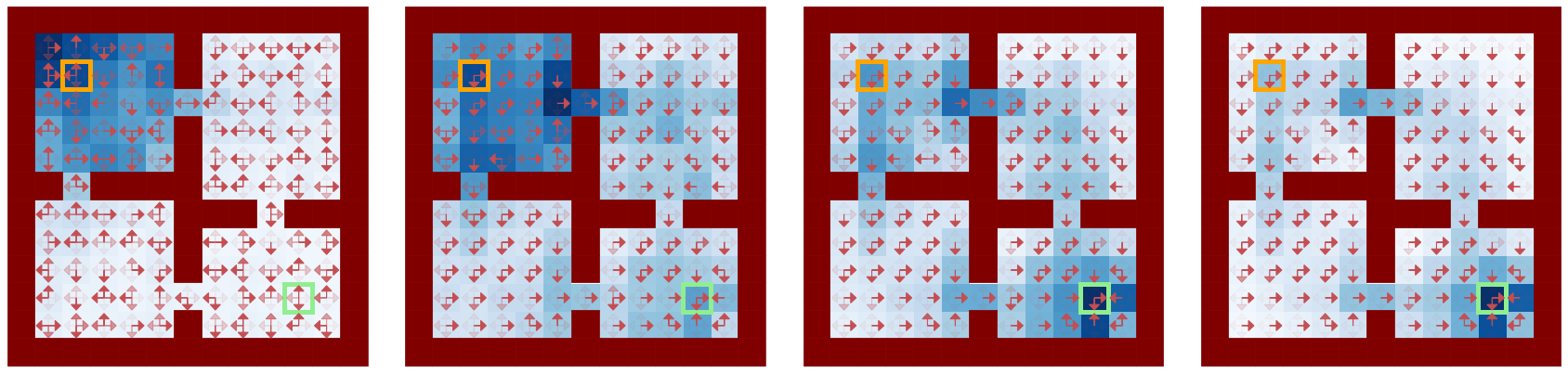} 
  \end{center}
  \caption{We provide a pictoral representation of learned policies $\pi$ and dual variables $\nu$ during training of \estabb on the Four Rooms domain~\citep{sutton1999between}.
  The agent is initialized at the state denoted by an orange square and receives zero reward everywhere except at the target state, denoted by a green square.
  We use a fixed offline experience data distribution that is near-uniform. The progression of learned $\pi$ and $\nu$ during training is shown from left to right.
  The policy $\pi$ is presented via arrows for each action at each square, with the opacity of the arrow determined by the probability $\pi(a|s)$. 
  The dual variables $\nu$ are presented via their Bellman residuals: the opacity of each square is
  determined by the sum of the Bellman residuals at that state $\sum_a (\bellman\nu - \nu)(s,a)$. Recall that for any $\pi$, the Bellman residuals of the optimal $\nustar_\pi$ should satisfy $(\bellman\nustar_\pi - \nustar_\pi)(s,a)=\w(s,a)$. 
As expected, we see that in the beginning of training the residuals are high near the initial state, while towards the end of training the residuals show the preferred trajectories of the near-optimal policy.
  }
  \label{fig:fourrooms}
  \vspace{10mm}
\end{figure}

\subsection{Four Rooms}
We begin by considering the tabular setting given by the Four Rooms environment~\citep{sutton1999between}, in which an agent must navigate to a target location within a gridworld.
In this tabular setting, we evaluate Primal~\estabb (Equations~\ref{eq:dice-nu} and~\ref{eq:thm1}) with $f(x)=\frac{1}{2}x^2$. 
This way, for any $\pi$, the dual value function $\nu$ may be solved exactly using standard matrix operations.
Thus, we train $\pi$ by iteratively solving for $\nu$ via matrix operations and then taking a gradient step for $\pi$.
We collect an off-policy dataset by running a uniformly random policy for 500 trajectories, where each trajectory is initalized at a random state and is of length 10.
This dataset is kept fixed, placing us in the completely offline regime.
We use $\alpha=0.01$ and $\gamma=0.97$.

\begin{wrapfigure}{R}{0.4\textwidth}
  \vspace{-5mm}
  \begin{minipage}{0.4\textwidth}
    \centering
    \includegraphics[width=0.99\columnwidth,height=0.8\columnwidth]{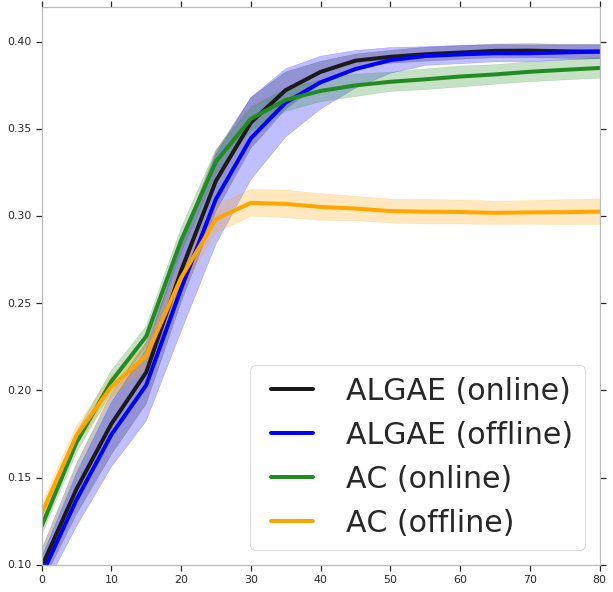}
    \vspace{-10mm}
    \caption{Average per-step reward of policies on Four Rooms learned by~\estabb compared to actor-critic (AC) over training iterations.}
    \label{fig:fourrooms-plot}
  \vspace{-4mm}
  \end{minipage}
\end{wrapfigure}

Graphical depictions of learned policies $\pi$ and dual value functions $\nu$ are presented in Figure~\ref{fig:fourrooms}, where each plot shows $\pi$ and $\nu$ during the first, fourth, seventh, and tenth iterations of training.
The opacity of each square is determined by the Bellman residuals of $\nu$ at that state. Recall that the Bellman residuals of the optimal $\nustar_\pi$ are the density ratios $\w$. 
We see that this is reflected in the learned $\nu$. At the beginning of training, the residuals are high around the initial state. As training progresses, there is a clear path (or paths) of high-residual states going from initial to target state.
Thus we see that $\nu$ learns to properly correct for distribution shifts in the off-policy experience distributions.
The algorithm successfully learns to optimize a policy using these corrected gradients, as shown by the arrows denoting preferred actions of the learned policy.

We further provide quantitative results in Figure~\ref{fig:fourrooms-plot}. We plot the average per-step reward of~\estabb compared to actor-critic in both online and offline settings.
As a point of comparison, the behavior policy used to collect data for the offline setting achieves average reward of $0.03$.
Although all the variants are able to significantly improve upon this baseline, we see that~\estabb performance is only negligibly affected by the type of dataset, while performance of actor-critic degrades in the offline regime.
See~\appref{app:details} for experimental details.

\subsection{Continuous Control}
We now present results of~\estabb on a set of continuous control benchmarks using MuJoCo~\citep{TodEreTas12} and OpenAI Gym~\citep{brockman2016openai}.
We evaluate the performance of Primal~\estabb with $f(x)=\frac{1}{2}x^2$.
Our empirical objective is thus given by
\begin{equation*}
  J(\pi,\nu)\defeq 2\alpha(1-\gamma)\cdot\E_{s_0\sim\Uset,a_0\sim\pi(s_0)}[\nu(s_0,a_0)] + \E_{(s,a,r,s')\sim\Dset,a'\sim\pi(s')}[\delta_{\nu,\pi}(s,a,r,s',a')^2],
\end{equation*}
where $\delta_{\nu,\pi}$ is a single-sample estimate of the Bellman residual:
\begin{equation}
  \label{eq:delta}
  \delta_{\nu,\pi}(s,a,r,s',a') \defeq r + \gamma\nu(s',a') - \nu(s,a).
\end{equation}
We note that using a single-sample estimate for the Bellman residual in general leads to biased gradients, although previous works have found this to not have a significant practical effect in these domains~\citep{valuedice}.
We make the following additional practical modifications:
\begin{itemize}
    \item As entropy regularization has been shown to be important on these tasks~\citep{tpcl,HaaZhoAbbLev18}, we augment the rewards with a causal entropy term; \ie, replace $r$ in~\eqref{eq:delta} with $r - \tau\log\pi(a'|s')$, where $\tau$ is learned adaptively as in~\citet{HaaZhoAbbLev18}.
    \item As residual learning is known to be hard in function approximation settings~\citep{baird1995residual}, we replace $\nu(s',a')$ in~\eqref{eq:delta} with a mixture $\eta\cdot\nu(s',a') + (1-\eta)\cdot\overline{\nu}(s',a')$ where $\overline{\nu}(s',a')$ is a target value calculated as in~\citet{HaaZhoAbbLev18}.
      We use $\eta=0.05$.
    \item If $\nu$ is fully optimized, $\delta$ will be the density ratio $\w$, and thus always non-negative. However, during optimization, this may not always hold, which can affect policy learning. Thus, when calculating gradients of this objective with respect to $\pi$, we clip the value of $\delta$ from below at 0.
\end{itemize}
For training we parameterize $\pi$ and $\nu$ using neural networks and perform alternating stochastic gradient descent on their parameters.

We present our results in Figure~\ref{fig:mujoco-results}.
We see that~\estabb can perform well in these settings, achieving performance that is roughly competitive with the state-of-the-art SAC and TD3 algorithms.
There are potentially more possible improvements to these practical results by choosing $f$ (or $\fstar$) appropriately. In~\appref{app:more-results}, we conduct a preliminary investigation into polynomial $f$, showing that certain polynomials can at times provide better performance than $f(x)=\frac{1}{2}x^2$. 
A more detailed and systematic study of this and other design choices for implementing~\estabb is an interesting avenue for future work.
\begin{figure}[t]
  \setlength{\tabcolsep}{0pt}
  \renewcommand{\arraystretch}{0.7}
  \begin{center}
  	\begin{tabular}{ccc}
      \tiny HalfCheetah & \tiny Hopper & \tiny Walker2d \\
    \includegraphics[width=0.27\textwidth]{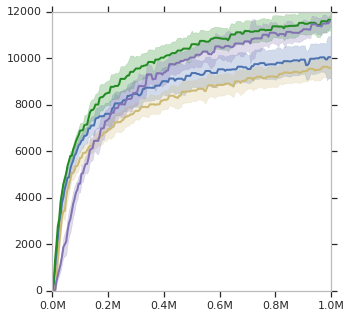} &
    \includegraphics[width=0.27\textwidth]{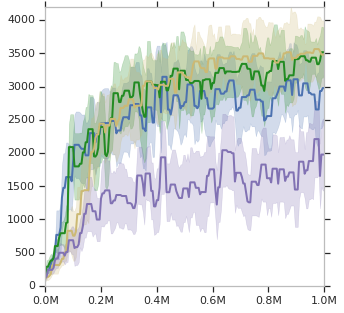} &
    \includegraphics[width=0.27\textwidth]{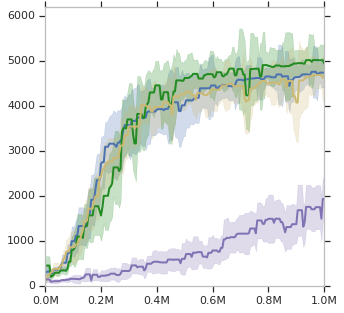}
    \end{tabular}
  	\begin{tabular}{cc}
      \tiny Ant & \tiny Humanoid \\
    \includegraphics[width=0.27\textwidth]{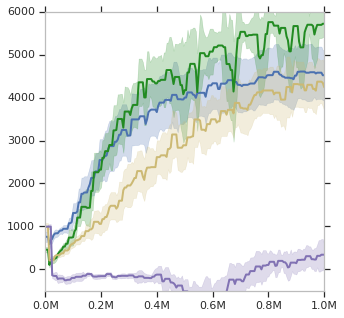} &
    \includegraphics[width=0.27\textwidth]{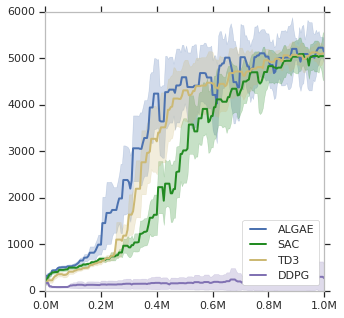}
    \end{tabular}
  \end{center}
  \caption{We show the results of~\estabb compared to SAC~\citep{HaaZhoAbbLev18}, TD3~\citep{fujimoto2018addressing}, and DDPG~\citep{lillicrap2015continuous}. 
  We follow the evaluation protocol of~\citet{fujimoto2018addressing}, plotting the performance of 10 randomly seeded training runs, with shaded region representing half a standard deviation and $x$-axis given by environment steps.
  There are potentially better results achievable by using a choice of $f$ other than $f(x)=\frac{1}{2}x^2$ for~\estabb; see~\appref{app:more-results} for a preliminary investigation.
  }
  \label{fig:mujoco-results}
  \vspace{10mm}
\end{figure}

\section{Conclusion}

We have introduced an~\estname, or \estabb, for \emph{behavior-agnostic}, \emph{off-policy} policy improvement in reinforcement learning.
Based on a linear programming characterization of the $Q$-function,
we derived the new approach from a Lagrangian saddle-point formulation.
The resulting algorithm,
\estabb, automatically compensates for the distribution shift 
in collected off-policy data, 
and achieves an estimate of the on-policy policy gradient 
using this off-policy data.

\subsubsection*{Acknowledgments}
We thank Marc Bellemare, Nicolas Le Roux, George Tucker, Rishabh Agarwal, Dibya Ghosh, and the rest of the Google Brain team for insightful thoughts and discussions.

\bibliographystyle{plainnat}
\bibliography{ref,../../../bibfile/bibfile}


\clearpage
\newpage

\appendix
\onecolumn

\begin{appendix}

\thispagestyle{plain}
\begin{center}
{\huge Appendix}
\end{center}

\section{Proof Details}\label{appendix:proof_details} 

We follow the notations in main text.  Abusing notation slightly, we will use $\sum$ and $\int$ interchangeably. 
\paragraph{\thmref{theorem:q_lp}}
{\itshape
 Given a policy $\pi$,
  the average return of $\pi$ may be expressed in the primal and dual forms as \\
  \begin{minipage}{0.45\textwidth}
\begin{align}
\min_{\nu:\S\times\A\rightarrow \Nset}& \jprimal(\pi, \nu) \defeq \rbr{1-\gamma}\EE_{\mu_0\pi}\sbr{\nu\rbr{s_0, a_0}}\nonumber \\
  \st~~~~ &\nu\rbr{s, a}\ge \bellman\nu(s, a), \qquad\quad \eqref{eq:q_lp} \nonumber \\
  & \forall \rbr{s, a}\in S\times A,\nonumber
\end{align}
\end{minipage}
~and,
\begin{minipage}{0.45\textwidth}
\begin{align}
  \max_{\rho:\S\times\A\rightarrow \RR_+}& \jdual(\pi, \rho) \defeq \EE_{\rho}\sbr{r\rbr{s, a}}\nonumber\\
  \st~~~~ &\rho\rbr{s, a} =  \bellmant\rho(s, a), \quad \eqref{eq:rho_lp} \nonumber \\
  & \forall \rbr{s, a}\in S\times A,\nonumber 
\end{align}
\end{minipage} \\
respectively. Under Assumptions~\ref{assumption:bounded_rewards} and \ref{assumption:mdp_reg}, strong duality holds, \ie, $\jprimal(\pi, \nustar_\pi) = \jdual(\pi, \rho^*_\pi)$ for optimal solutions $\nustar_\pi,\rho^*_\pi$. 
  The optimal primal satisfies $\nustar_\pi\rbr{s, a} = \qpi\rbr{s, a}$ for all $\rbr{s, a}$ reachable by $\pi$ 
  and the optimal dual $\rho^*_\pi$ is $\visitpi$.

}


\begin{proof}
Recall that $\bellman$ is monotonic; that is, given two bounded functions $\nu_1$ and $\nu_2$, $\nu_1 \ge \nu_2$ implies $\bellman\nu_1 \ge \bellman \nu_2$.
Therefore, for any feasbile $\nu$, we have $\nu\ge \rbr{\bellman} \nu\ge \rbr{\bellman}^2 \nu \ge\rbr{\bellman}^3 \nu\ge\ldots\ge \rbr{\bellman}^\infty \nu = Q_\pi$, proving the first claim. 

The duality of the linear program~\eqref{eq:q_lp} can be obtained as
\begin{eqnarray*} 
\max_{\rho:\S\times\A\rightarrow\RR_+}&& \EE_{\rho}\sbr{r\rbr{s, a}} \nonumber \\
\st && \rho\rbr{s', a'}  = \underbrace{\gamma\sum_{s, a}{ \pi\rbr{a'|s'}T\rbr{s'|s, a}\rho\rbr{s, a}} + \rbr{1 - \gamma}\mu_0\rbr{s'}\pi\rbr{a'|s'}}_{\bellmant\rho\rbr{s', a'}},\\
&&\forall \rbr{s', a'}\in \S\times\A \,, \nonumber
\end{eqnarray*}
which is exactly \eqref{eq:rho_lp}.
Notice that the equality constraints correspond to a system of linear equations of dimension $\abr{\S}\times\abr{\A}$:
$(I - \gamma \rbr{P_\pi}^\top) \rho = (1-\gamma)(\mu_0\pi)$,
where $(\mu_0\pi)(s,a) = \mu_0(s)\pi(a|s)$,
$P_\pi\rbr{s', a'|s, a} = \pi\rbr{a'|s'}T\rbr{s'|s, a}$,
and $I$ is the identity matrix.
Since the matrix $I - \gamma \rbr{P_\pi}^\top$ is nonsingular, the system has a unique solution given by
\begin{equation*} 
\rho^* = \rbr{1 - \gamma}\rbr{I - \gamma \rbr{P_\pi}^\top}^{-1}\rbr{\mu_0\pi}\,.
\end{equation*}
Finally, when $\gamma\in [0, 1)$, we can rewrite $\rbr{I - \gamma \rbr{P_\pi}^\top}^{-1} =  \sum_{t=0}^\infty \gamma^t \rbr{P_\pi}^t$, so $\rho^* = \rbr{1 - \gamma}\sum_{t=0}^\infty \gamma^t \rbr{P_\pi}^t\rbr{\mu_0\pi} = \visitpi$, as desired.
\end{proof}


\paragraph{\thmref{thm:effect_reg}}
{\itshape
Under Assumptions~\ref{assumption:bounded_rewards}--\ref{assumption:convex_f},
the solution to \eqref{eq:q_lagrangian_reg} is given by,
\begin{eqnarray*}
  \nustar_\pi\rbr{s, a} &=& -\alpha f'\rbr{\w\rbr{s, a}}+\bellman\nustar_\pi\rbr{s,a}, \\
  \zetastar_\pi\rbr{s, a} &=& \w\rbr{s, a}.
\end{eqnarray*}
The optimal value is $L\rbr{\nustar_\pi, \zetastar_\pi;\pi} = \E_{\visitpi}[r(s,a)]-\alpha\fdiv(\visitpi \| \visitrb)$. 
}
\begin{proof}
By Fenchel duality, we have
\begin{align*}
\max_{\zeta:\Scal\times\Acal\rightarrow \RR}\,\, & \EE_{\visitrb}\sbr{\zeta\rbr{s, a} {\rbr{\bellman \nu - \nu}\rbr{s, a}}}- \alpha\,\EE_{\visitrb}\sbr{f\rbr{\zeta\rbr{s, a}}} \\
&= \alpha\,\EE_{\visitrb}\sbr{\f_*\rbr{\frac{1}{\alpha}\rbr{\bellman \nu - \nu}\rbr{s, a}}}.
\end{align*}
Plugging this into \eqref{eq:q_lagrangian_reg}, we have
  \begin{equation}
    \label{eq:l-zstar}
  L(\nu,\zetastar_\pi;\pi)=
\min_{\nu:\Scal\times\Acal\rightarrow \Nset} \,\, \rbr{1 - \gamma}\EE_{\mu_0\pi}\sbr{\nu\rbr{s_0, a_0}} + \alpha\EE_{\visitrb}\sbr{f_*\rbr{\frac{1}{\alpha}\rbr{\bellman \nu - \nu}\rbr{s, a}}}.
  \end{equation}
  To investigate the optimality, we apply the change-of-variable, $x\rbr{s, a} \defeq \frac{1}{\alpha} \rbr{\bellman \nu - \nu}\rbr{s, a}$.  
  Let $\beta_t\rbr{s} = P\rbr{s = s_t|s_0\sim \mu_0, \cbr{a_i}_{i=0}^t\sim \pi}$, 
  and consider the first expectation in~\eqref{eq:l-zstar}:
  \begin{flalign*}
    (1-\gamma)\EE_{\init\pi}[\nu(s_0,a_0)] &&
  \end{flalign*}
\begin{eqnarray*}
  ~~ &=& (1-\gamma)\sum_{t=0}^\infty \gamma^t \EE_{s\sim\beta_t,a\sim\pi(s)}[\nu(s,a)] - (1-\gamma)\sum_{t=0}^\infty \gamma^{t+1}\EE_{s'\sim\beta_{t+1},a'\sim\pi(s')}[\nu(s',a')] \\
  &=& \EE_{(s,a)\sim\visitpi}[\nu(s,a) - \gamma \EE_{s'\sim T(s,a),a'\sim\pi(s')}[\nu(s',a')]] \\
  &=& \EE_{(s,a)\sim\visitpi}[r(s,a)] + \E_{(s,a)\sim\visitpi}[\nu(s,a) - r(s,a) - \gamma \EE_{s'\sim T(s,a),a'\sim \pi(s')}[\nu(s',a')]] \\ 
  &=& \EE_{(s,a)\sim\visitpi}[r(s,a)] - \alpha \EE_{(s,a)\sim\visitpi}[x(s,a)].
\end{eqnarray*}
  Let $\mathcal{C}$ denote the set of functions $x$ in the image of $\rbr{\bellman \nu - \nu}$ for $\nu:\S\times\A\to\Nset$. 
  Therefore, the change of variables yields the following re-formulation of $L$:
\begin{eqnarray*}
  L(\nustar_\pi,\zetastar_\pi;\pi) &=&\min_{x\in\mathcal{C}} \EE_{(s,a)\sim\visitpi}\sbr{r\rbr{s, a}} - \alpha \EE_{(s,a)\sim\visitpi}\sbr{x\rbr{s, a}} + \alpha \EE_{\visitrb}\sbr{f_*\rbr{x\rbr{s, a}}}\\
  &=&\EE_{(s,a)\sim\visitpi}[r(s,a)] - \alpha\rbr{\max_{x\in \mathcal{C}} \EE_{\visitpi}\sbr{x(s, a)} - \EE_{\visitrb}\sbr{f_*\rbr{x\rbr{s, a}}}}
\end{eqnarray*}
  Note that, ignoring the restriction of $x$ to $\mathcal{C}$ (for now),  the optimal $\xstar_\pi$ satisfies $\fstar'(\xstar_\pi(s,a))=\w(s,a)$. 
  By Assumption~\ref{assumption:convex_f}, we have that $\sbr{\rbr{f_*}'}^{-1}\rbr{\cdot}$ exists, and equals $f'(\cdot)$. Thus, we have $\xstar_\pi(s,a) = f'(\w(s,a))$ for all $s,a$.
  Due to the~\asmpref{assumption:bounded_ratios} that $\w$ is bounded, we have that $\xstar_\pi$ is bounded by ${f'(\Wmax)}$ and thus $\xstar_\pi\in\mathcal{C}$.  
  Therefore, by definition of the $f$-divergence, we have
\begin{equation}
  L(\nustar_\pi,\zetastar_\pi;\pi) = \EE_{(s,a)\sim\visitpi}[r(s,a)] - \alpha\fdiv(\visitpi || \visitrb),
\end{equation}
as desired.

To characterize $\nustar_\pi$, we note,
\begin{equation}
  x^*\rbr{s, a} = f'(\w(s,a)) \Rightarrow \nustar_\pi\rbr{s, a} = \bellman \nustar_\pi(s,a) -  \alpha f'(\w(s,a)).
\end{equation}

  To characterize the optimal dual $\zetastar_\pi\rbr{s, a}$, we have
$$
  \zetastar_\pi\rbr{s, a} = 
  \argmax_\zeta \zeta \cdot \xstar_\pi(s,a) - f(\zeta) = \fstar'(\xstar_\pi(s,a)) = \w(s,a)
$$
  where the second equality comes from the fact that $f'(\zetastar_\pi(s,a))=\xstar_\pi(s,a)\Rightarrow \zetastar_\pi(s,a) =\fstar'(\xstar_\pi(s,a))$. 

\end{proof}

\subsection{Extension to $\gamma=1$}\label{appendix:proof_details_undiscounted} 

We follow similar steps in the previous section, and extend the analysis to the undiscounted case.
Different from the discounted case, the primal variable has an extra scalar component, denoted $\lambda$.
Under~\asmpref{assumption:mdp_reg}, the Markov chain induced by $\pi$ is ergodic, with a unique invariant distribution $\visitpi$ and mixing time $\tmix$.  The mixing time quantifies the number of steps for the state distribution in the induced Markov chain to be close to $\visitpi$, measured by total variation~\citep{LevPer17}.  Precisely,
\[
\tmix \defeq \min\cbr{t ~:~ \sup_{(s,a)\in\SAset}\tvar{\delta_{s,a} T_\pi^t - \visitpi} \le 1/4}\,,
\]
where $\delta_{s,a}$ is the delta measure concentrated on the state-action pair $(s,a)$, and $T_\pi\rbr{s', a'|s, a} \defeq T\rbr{s'|s, a}\pi\rbr{a'|s'}$.  The range of the primal variables $\nu$ is changed to $\Nset \defeq [-C\tmix,C\tmix]$ where $C=2\Rmax + |\alpha| \cdot (\f'(\Wmax)-\f'(0))$.

We begin with the $Q$-LP characterization of $\qpi$-values and visitations $\visitpi$.
\begin{theorem} \label{theorem:q_lp_ud}
Under \asmpref{assumption:mdp_reg},
the average return of $\pi$ may be expressed in the primal and dual forms as \\
\begin{minipage}{0.45\textwidth}
\begin{align}
\min_{\lambda\in\R,\nu:\S\times\A\rightarrow \Nset}& \jprimal(\pi, \nu, \lambda) \defeq \lambda \nonumber \\
  \st~~~~ &\nu\rbr{s, a}\ge -\lambda + \bellman\nu(s, a), \\
  & \forall \rbr{s, a}\in S\times A,\nonumber
\end{align}
\end{minipage}
~and,
\begin{minipage}{0.45\textwidth}
\begin{align}
  \max_{\rho:\S\times\A\rightarrow \RR_+}& \jdual(\pi, \rho) \defeq \EE_{\rho}\sbr{r\rbr{s, a}}\nonumber\\
  \st~~~~ &\rho\rbr{s, a} =  \bellmant\rho(s, a), \\
  & \forall \rbr{s, a}\in S\times A,\nonumber \\
  & \sum\rho\rbr{s,a} = 1\,, \nonumber
\end{align}
\end{minipage} \\
respectively. Under further \asmpref{assumption:bounded_rewards}, strong duality holds, \ie, $\jprimal(\pi, \lambdastar_\pi, \nustar_\pi) = \jdual(\pi, \rho^*_\pi)$ for optimal solutions $(\lambdastar_\pi,\nustar_\pi,\rho^*_\pi)$. 
As in the discounted case, we have $\rho^*_\pi = \visitpi$.
Unlike the discounted case, there are infinitely many solutions for $\nustar_\pi$, as any optimal solution for $\nu$ remains optimal with a constant offset.
\end{theorem}

Given these LP formulations, the Primal and Fenchel~\estabb optimization problems for $\gamma=1$ are given by
\begin{equation}
  \max_\pi\min_{\lambda\in\R,\nu:\Sset\times\Aset\to\Nset} \jdice(\pi,\lambda,\nu) \defeq \lambda + 
  \alpha\cdot\EE_{(s,a)\sim\visitrb}[\fstar((-\lambda + \bellman \nu(s,a) - \nu(s,a)) / \alpha)],
\end{equation}
and
\begin{multline}
  \label{eq:undisc-lagrange}
  \max_\pi\min_{\lambda\in\R,\nu:\Sset\times\Aset\to\Nset}\max_{\zeta:\Sset\times\Aset\to\RR} L\rbr{\lambda, \nu, \zeta; \pi}\defeq \lambda + \\ \EE_{(s,a)\sim\visitrb}\sbr{\zeta(s, a)(-\lambda + \bellman \nu(s,a) - \nu(s,a))} - \alpha\cdot\EE_{(s,a)\sim\visitrb}\sbr{f\rbr{\zeta\rbr{s, a}}},
\end{multline}
respectively.  We will show in~\thmref{thm:effect_reg_ud} that the optimal $\zeta^*$ in Fenchel~\estabb automatically satisfies the constraint $\zeta\in \Zset \defeq \cbr{\zeta\ge 0, \EE_{d^{\Dcal}}\sbr{\zeta\rbr{s, a}} = 1}$, avoiding the nontrivial self-normalization step in \citet{liu2018breaking}.
In fact, by comparing Fenchel~\estabb with GenDICE~\citep{zhang2020gendice}, the GenDICE objective with unit penalty weight could be understood as primal variables regularized Lagrangian of $Q$-LP, while the Fenchel~\estabb is derived by regularizing the dual variable in the Lagrangian of $Q$-LP. 

We have the following analogue to Theorem~\ref{thm:effect_reg}.
\begin{theorem} \label{thm:effect_reg_ud}
Under Assumptions~\ref{assumption:bounded_rewards}--\ref{assumption:convex_f}, the solution to
  \eqref{eq:undisc-lagrange} is given by,
\begin{eqnarray*}
  \nustar_\pi\rbr{s, a} &=& -\lambdastar_\pi-\alpha f'\rbr{\w\rbr{s, a}}+\bellman\nustar_\pi(s,a), \\
  \zetastar_\pi\rbr{s, a} &=& \w\rbr{s, a}.
\end{eqnarray*}
The optimal value is $L\rbr{\lambdastar_\pi, \nustar_\pi, \zetastar_\pi;\pi} = \E_{\visitpi}[r(s,a)]-\alpha\fdiv(\visitpi \| \visitrb)$. 
\end{theorem}
\begin{proof}
By Fenchel duality, we have
  \begin{multline}
    \max_{\zeta:\Scal\times\Acal\rightarrow \RR}\,\, \EE_{\visitrb}\sbr{\zeta\rbr{s, a} (-\lambda + \bellman \nu(s,a) - \nu(s,a))}- \alpha\EE_{\visitrb}\sbr{f\rbr{\zeta\rbr{s, a}}} \\ = \alpha\EE_{\visitrb}\sbr{\f_*\rbr{\frac{1}{\alpha}\rbr{-\lambda + \bellman \nu(s,a) - \nu(s,a)}}}.
  \end{multline}
Plugging this into \eqref{eq:undisc-lagrange}, we have
\begin{equation} \label{eq:l-zstar-ud}
L(\lambda,\nu,\zetastar_\pi;\pi)=
  \lambda + \alpha \, \EE_{\visitrb}\sbr{f_*\rbr{\frac{1}{\alpha}\rbr{-\lambda + \bellman \nu(s,a) - \nu(s,a)}}}.
\end{equation}
To investigate the optimality, we apply the change-of-variable, 
  \begin{equation*}
    x\rbr{s, a} \defeq \frac{1}{\alpha} \rbr{-\lambda + \bellman \nu(s,a) - \nu(s,a)}.
  \end{equation*}
We have,
\begin{eqnarray}
\EE_{\visitpi}\sbr{x(s,a)}
  &=& \alpha^{-1} \EE_{\visitpi}\sbr{-\lambda + \bellman\nu(s,a) - \nu(s,a)} \nonumber \\
  &=& \alpha^{-1} \EE_{\visitpi}\sbr{\rbr{P_\pi \nu - \nu}\rbr{s,a} + r\rbr{s,a} - \lambda} \nonumber \\
  &=& \alpha^{-1} \EE_{\visitpi}\sbr{r\rbr{s,a} - \lambda}\,, \label{eq:x-lambda}
\end{eqnarray}
where the last equality holds because
\[
  \EE_{\visitpi}\sbr{\rbr{P_\pi \nu - \nu}\rbr{s,a}}
= \EE_{\visitpi}\sbr{P_\pi \nu\rbr{s,a}} - \EE_{\visitpi}\sbr{\nu\rbr{s,a}}
= \EE_{\visitpi}\sbr{\nu\rbr{s,a}} - \EE_{\visitpi}\sbr{\nu\rbr{s,a}}
= 0\,.
\]
Therefore,
\begin{eqnarray*}
  L(\lambda, \nu,\zetastar_\pi;\pi)
&=& \lambda + \EE_{\visitpi}\sbr{r\rbr{s,a} - \lambda} - \EE_{\visitpi}\sbr{r\rbr{s,a} - \lambda}\\
& & \,\,\,\,\, + \, \alpha\EE_{\visitrb}\sbr{f_*\rbr{\frac{1}{\alpha}\rbr{-\lambda + \bellman \nu(s,a) - \nu(s,a)}}} \\
&=& \EE_{\visitpi}\sbr{r\rbr{s,a}} - \alpha \EE_{\visitpi}\sbr{x(s,a)} + \alpha\EE_{\visitrb}\sbr{f_*\rbr{x\rbr{s,a}}}\,.
\end{eqnarray*}

Let $\mathcal{C}$ denote the set of functions $x$ in the image of $-\lambda + \rbr{\bellman \nu - \nu}$ for $\nu:\S\times\A\to\Nset$. 
Therefore, the change of variables yields the following re-formulation of $L$:
\begin{eqnarray*}
  L(\lambda^*_\pi, \nustar_\pi,\zetastar_\pi;\pi)
&=&\min_{x\in\mathcal{C}} \EE_{(s,a)\sim\visitpi}\sbr{r\rbr{s, a}} - \alpha \EE_{(s,a)\sim\visitpi}\sbr{x\rbr{s, a}} + \alpha \EE_{\visitrb}\sbr{f_*\rbr{x\rbr{s, a}}}\\
&=&\EE_{(s,a)\sim\visitpi}[r(s,a)] - \alpha\rbr{\max_{x\in \mathcal{C}} \EE_{\visitpi}\sbr{x(s, a)} - \EE_{\visitrb}\sbr{f_*\rbr{x\rbr{s, a}}}}\,.
\end{eqnarray*}
  Note that, ignoring the restriction of $x$ to $\mathcal{C}$ (for now),  the optimal $\xstar_\pi$ satisfies $\fstar'(x(s,a))=\w(s,a)$. 
  By Assumption~\ref{assumption:convex_f}, we have that $\sbr{\rbr{f_*}'}^{-1}\rbr{\cdot}$ exists, and it is given by $f'(\cdot)$. Thus, we have $\xstar_\pi(s,a) = f'(\w(s,a))$ for all $s,a$.
  Due to the~\asmpref{assumption:bounded_ratios} that $\w$ is bounded, we have that $\xstar_\pi$ is bounded by ${f'(\Wmax)}$ and thus $\xstar_\pi\in\mathcal{C}$.  
  Therefore, by definition of the $f$-divergence, we have
\begin{equation}
  \label{eq:undisc-opt-value}
  L(\lambdastar_\pi,\nustar_\pi,\zetastar_\pi;\pi) = \EE_{(s,a)\sim\visitpi}[r(s,a)] - \alpha\fdiv(\visitpi || \visitrb),
\end{equation}
as desired.

To characterize $\nustar_\pi$, we note,
\begin{equation}
  x^*\rbr{s, a} = f'(\w(s,a)) \Rightarrow \nustar_\pi\rbr{s, a} = -\lambdastar_\pi + \bellman \nustar_\pi(s,a) -  \alpha f'(\w(s,a)).
\end{equation}

To characterize the optimal dual $\zetastar_\pi\rbr{s, a}$, we have
$$
  \zetastar_\pi\rbr{s, a} = 
  \argmax_\zeta \zeta \cdot \xstar_\pi(s,a) - f(\zeta) = \fstar'(\xstar_\pi(s,a)) = \w(s,a)
$$
  where the second equality comes from the fact that $f'(\zetastar_\pi(s,a))=\xstar_\pi(s,a)\Rightarrow \zetastar_\pi(s,a) =\fstar'(\xstar_\pi(s,a))$. 

The final step is to show $\lambdastar_\pi$ is bounded, and there exists an optimal solution $\nustar_\pi$ whose image is in $\Nset$.  
%
%
As discussed, the $x^*\rbr{s, a} = f'(\w(s,a))$ is bounded by $f'\rbr{W_{\max}}$, then, we may use~\eqref{eq:x-lambda} to characterize $\lambdastar_\pi$ as
$$
\lambdastar_\pi = \EE_{\visitpi}\sbr{r\rbr{s,a}} - \alpha\EE_{\visitpi}\sbr{x^*(s,a)} \,\Rightarrow \abr{\lambda}\le C <\infty.
$$
We now consider an optimal solution $\nustar_\pi$, and let $(\hat{s},\hat{a}) = \argmax\nustar(s,a)$ and $(\bar{s},\bar{a}) = \argmin\nustar(s,a)$.\footnote{The analysis here can be adapted to the case where $\max\nustar$ and $\inf\nustar$ are replaced by $\sup\nustar$ and $\inf\nustar$, respectively.}  As the set of $\nustar_\pi$ is offset-invariant, we assume without loss of generality that $\nustar(\bar{s},\bar{a})=0$.
Consider a trajectory starting from $(\hat{s},\hat{a})$ and controlled by $\pi$, $(\hat{s}_0,\hat{a}_0,\hat{s}_1,\hat{a}_1,\ldots)$, on which repeated applications of $\bellman$ yield
\[
\nustar_\pi(\hat{s},\hat{a}) = \EE_{\hat{s}_0=\hat{s},\hat{a}_0=\hat{a},\pi}\sbr{\sum_{t<\tmix} \left(r(\hat{s}_t,\hat{a}_t) - \alpha f'(\w(\hat{s}_t,\hat{a}_t)) - \lambdastar_\pi\right) + \nustar_\pi(\hat{s}_{\tmix},\hat{a}_{\tmix})}\,.
\]
We may obtain a similar recurrence for a trajectory, $(\bar{s}_0,\bar{a}_0,\bar{s}_1,\bar{a}_1,\ldots)$, starting from $(\bar{s},\bar{a})$. Subtracting them on both sides, we have
\begin{eqnarray}
\lefteqn{\nustar_\pi(\hat{s},\hat{a}) - \nustar_\pi(\bar{s},\bar{a})
 = \EE_{\hat{s}_0=\hat{s},\hat{a}_0=\hat{a},\pi}\sbr{\sum_{t<\tmix} \left(r(\hat{s}_t,\hat{a}_t) - \alpha f'(\w(\hat{s}_t,\hat{a}_t))\right)}} \nonumber \\
 & & - \EE_{\bar{s}_0=\bar{s},\bar{a}_0=\bar{a},\pi}\sbr{\sum_{t<\tmix} \left(r(\bar{s}_t,\bar{a}_t) - \alpha f'(\w(\bar{s}_t,\bar{a}_t))\right)} \nonumber \\
 & & + \left(\EE_{\hat{s}_0=\hat{s},\hat{a}_0=\hat{a},\pi}\sbr{\nustar_\pi(\hat{s}_{\tmix},\hat{a}_{\tmix})} - \EE_{\bar{s}_0=\bar{s},\bar{a}_0=\bar{a},\pi}\sbr{\nustar_\pi(\bar{s}_{\tmix},\bar{a}_{\tmix})}\right)\,. \label{eqn:span-1}
\end{eqnarray}
Consider the last term above.  By the definition of $\tmix$ and nonnegativity of $\nustar_\pi$, we have
\begin{eqnarray*}
\lefteqn{\EE_{\hat{s}_0=\hat{s},\hat{a}_0=\hat{a},\pi}\sbr{\nustar_\pi(\hat{s}_{\tmix},\hat{a}_{\tmix})} - \EE_{\bar{s}_0=\bar{s},\bar{a}_0=\bar{a},\pi}\sbr{\nustar_\pi(\bar{s}_{\tmix},\bar{a}_{\tmix})}} \\
 &\le& \tvar{\delta_{\hat{s},\hat{a}} T_\pi^{\tmix} - \delta_{\bar{s},\bar{a}} T_\pi^{\tmix}} \,\, \nustarmax \\
 &\le& (\tvar{\delta_{\hat{s},\hat{a}} T_\pi^{\tmix} - \visitpi} + \tvar{\delta_{\bar{s},\bar{a}} T_\pi^{\tmix} - \visitpi}) \nustarmax = \nustarmax/2\,. 
\end{eqnarray*}
Plugging the above in \eqref{eqn:span-1} and realizing $\nustar_\pi(\bar{s},\bar{a})=0$, we obtain
\[
\nustarmax/2 \le \tmix(2\Rmax + |\alpha|(f'(\Wmax)-f'(0))) = C \tmix\,.
\]
Hence, $\nustar$ is in the range $[0,2C\tmix]$.  Then, by the offset-invariance property, we can shift $\nustar$ by the constant $-C\tmix$, so that its range is now in $\Nset=[-C\tmix,C\tmix]$.
\end{proof}

\end{appendix}

\section{Experiment Details}

\label{app:details}
For the Four Rooms quantitative results, we used $100$ trajectories of length $100$.
For online training, these trajectories were sampled anew during each iteration.
For offline training, the trajectories were sampled once, using the behavior policy for the GridWalk environment from~\citet{dualdice}, and this dataset was fixed throughout training.
We implemented actor-critic analogous to the tabular version of~\estabb. 
At each training iteration, we first solve for $\qpi$ in closed form using standard matrix operations, and subsequently we take one gradient step for $\pi$ to maximize $\E_{s\sim \Dset}[\sum_a \pi(a|s) \qpi(s,a)]$.
The discount factor was $0.99$.

The table below gives hyperparameters used in the continuous control experiments.
Many of our settings for~\estabb were taken from~\citet{HaaZhoAbbLev18} and~\citet{fujimoto2018addressing}.
We set the regularization coefficient $\alpha$ in~\estabb to $0.01$. 

\begin{center}
\begin{tabular}{ c || c | c | c | c}
 Hyperparameter & DDPG & TD3 & \estabb & SAC \\ 
 \hline
 Policy network size & 400-300 & 400-300 & 256-256 & 256-256 \\  
 Critic network size & 400-300 & 400-300 & 256-256 & 256-256 \\  
Actor learning rate & $10^{-3}$ & $10^{-3}$ & $10^{-3}$ & $10^{-3}$ \\
Critic learning rate & $10^{-3}$ & $10^{-3}$ & $10^{-3}$ & $10^{-3}$ \\
Batch size & 256 & 100 & 256 & 256 \\
Critic updates per time step & 1 & 1 & 1 & 1\\
Critic updates per actor update & 1 & 2 & 2 & 2\\
\end{tabular}
\end{center}

\section{Additional Results}
\label{app:more-results}

\begin{figure}[H]
  \begin{center}
  \includegraphics[width=0.99\textwidth]{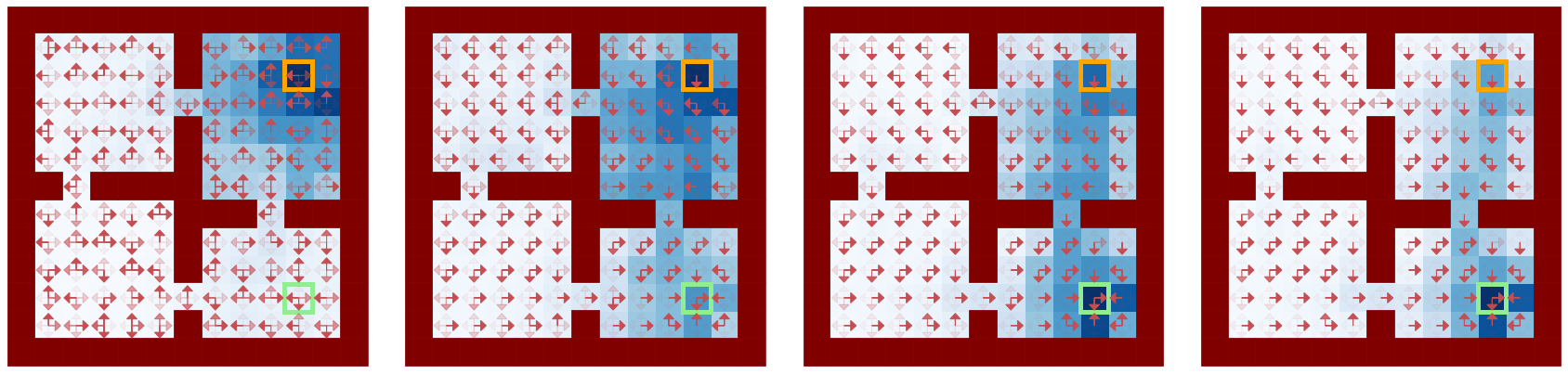} \\
  \includegraphics[width=0.99\textwidth]{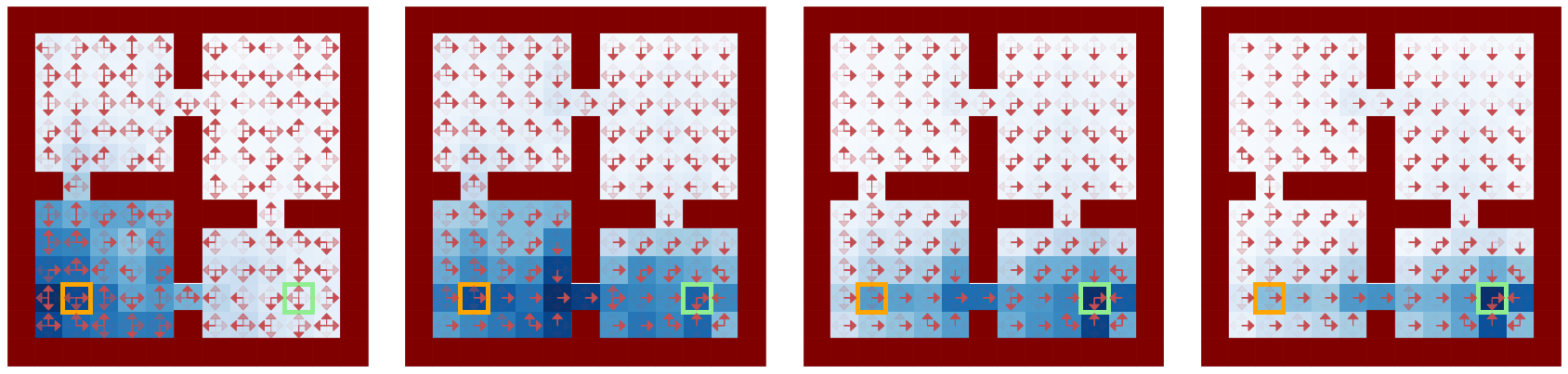}
  \end{center}
  \caption{Results on the Four Rooms domain with other initial states.
  }
  \label{fig:more-fourrooms}
\end{figure}

\begin{figure}[H]
  \setlength{\tabcolsep}{0pt}
  \renewcommand{\arraystretch}{0.7}
  \begin{center}
  	\begin{tabular}{ccc}
      \tiny HalfCheetah & \tiny Hopper & \tiny Walker2d \\
    \includegraphics[width=0.27\textwidth]{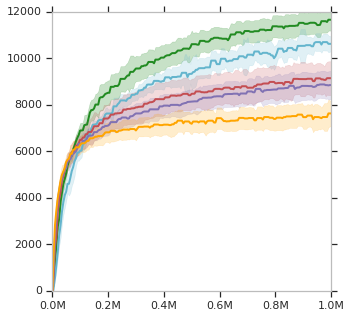} &
    \includegraphics[width=0.27\textwidth]{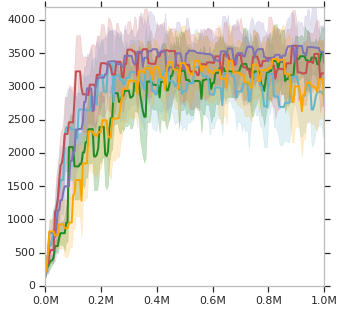} &
    \includegraphics[width=0.27\textwidth]{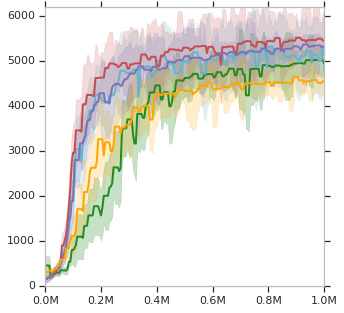}
    \end{tabular}
  	\begin{tabular}{cc}
      \tiny Ant & \tiny Humanoid \\
    \includegraphics[width=0.27\textwidth]{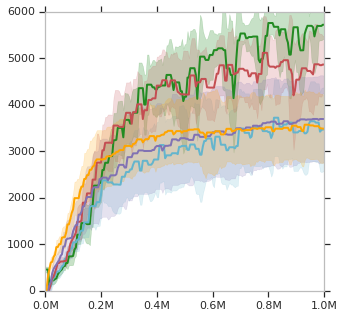} &
    \includegraphics[width=0.27\textwidth]{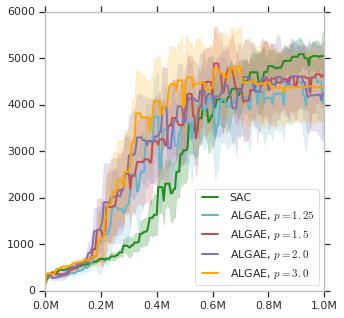}
    \end{tabular}
  \end{center}
  \caption{We show the results of~\estabb over the choice of function $\fstar(x)=\frac{1}{p}|x|^p$. In~\estabb runs, we use $\alpha=1$ and perform 4 training steps per environment step (as opposed to $\alpha=0.01$ and 1 training step per environment step used in the main text). We see that different values of $p$ lead to slightly different results. Interestingly, we find $p=1.5$ to typically perform the best, similar to the findings in~\citet{dualdice}.
  }
  \label{fig:more-results}
\end{figure}

\end{document}